\documentclass[runningheads]{llncs} 

\usepackage{amssymb}
\usepackage{graphicx}
\usepackage{cancel}

\usepackage{url}

\usepackage{pgf} 

\usepackage{tikz}
\usetikzlibrary{automata,arrows,positioning,snakes,shapes,backgrounds,fit,decorations,circuits.ee.IEC}

\newcommand{\ddist}{4.5pt}

\tikzset{every picture/.style={inner sep=1.5pt}}
\tikzstyle{t}=[isosceles triangle,draw,anchor=east,inner sep=0pt,rotate=90]
\tikzstyle{w}=[draw,fill,circle]
\tikzstyle{d}=[draw,circle,minimum size=2*\ddist,after node path={node[circle,fill] at (\tikzlastnode)  {}}]
\tikzstyle{wl}=[draw,fill,circle,after node path={edge [out=145,in=-145,loop] ()}]
\tikzstyle{ws}=[draw,circle]

\usepackage{comment}
\usepackage{amsmath}

\newcommand{\A}{ {\mathcal A} }

\newcommand{\X}{ {\mathcal X} }
\newcommand{\D}{ {\mathsf D} }
\newcommand{\stream}{ {\mathcal E} }

\DeclareMathOperator{\Set}{set}

\newcommand{\Nat}{ {\mathbb{N}} }

\newcommand{\m}{\mathcal{M}}

\spnewtheorem*{proofsketch}{Proof sketch}{\itshape}{\rmfamily}
\if@envcntsame 
   \def\spn@wtheorem#1#2#3#4{\@spothm{#1}[theorem]{#2}{#3}{#4}}
\else 
   \if@envcntsect 
      \def\spn@wtheorem#1#2#3#4{\@spxnthm{#1}{#2}[section]{#3}{#4}}
   \else 
      \if@envcntreset
         \def\spn@wtheorem#1#2#3#4{\@spynthm{#1}{#2}{#3}{#4}
                                   \@addtoreset{#1}{section}}
      \else
         \def\spn@wtheorem#1#2#3#4{\@spynthm{#1}{#2}{#3}{#4}
                                   \@addtoreset{#1}{chapter}}%
      \fi
   \fi
\fi

\newcommand{\Obs}{{\mathcal{O}}}

\def\epislang{\mathcal{L}_{epis}}
\def\proplang{\mathcal{L}_{prop}}
\def\M{\mathcal{M}}

\newcommand{\tb}[1]{{\bfseries\color{cyan}[Thomas: #1]}}
\newcommand{\nina}[1]{{\bfseries\color{magenta}[Nina: #1]}}

\newcommand{\journal}[1]{}

\newcommand{\full}[1]{#1}
\newcommand{\short}[1]{}

\def\Dom{dom}
\newcommand{\bisim}{{\raisebox{.3ex}[0mm][0mm]{\ensuremath{\medspace \underline{\! \leftrightarrow\!}\medspace}}}}
\def\actionspace{\mathsf{Actions}}
\def\actionnames{A}
\def\a{\mathcal{A}}
\def\actlib{l}
\def\maybeneg{\raisebox{1pt}{\scalebox{0.6}{$($}}\neg\raisebox{1pt}{\scalebox{0.6}{$)$}}}

\begin{document}

\mainmatter 

\title{Learning Actions Models: Qualitative Approach}
\author{Thomas Bolander \and Nina Gierasimczuk}
\institute{}
\maketitle

\begin{abstract}
In dynamic epistemic logic, actions are described using action models. In this paper we introduce a framework for studying learnability of action models from observations. We present first results concerning propositional action models. First we check two basic learnability criteria: finite identifiability (conclusively inferring the appropriate action model in finite time) and identifiability in the limit (inconclusive convergence to the right action model). We show that deterministic actions are finitely identifiable, while non-deterministic actions require more learning power---they are identifiable in the limit. We then move on to a particular learning method, which proceeds via restriction of a space of events within a learning-specific action model. This way of learning closely resembles the well-known update method from dynamic epistemic logic. We introduce several different learning methods suited for finite identifiability of particular types of deterministic actions. 
\end{abstract}
\journal{\tb{Something to be introduced fairly early is clarifying our priorities in this paper. Our priorities in decreasing order is that we wish to say things about: 1) finite identifiability/identifiability in the limit/non-identifiability; 2) time complexities; 3) space complexities; 4) minimal action descriptions. We will not have a lot to say about 3 and 4 in this paper, in fact we have prioritized to keep the learning algorithms as simple as possible, even though this implies that some of them have much worse space complexities than more efficient algorithms solving the same problem (we start out with an initial hypothesis containing anything conceivable and then we filter it down). I think it is important to get this message across, so that readers/reviewers don't start complaining. In many cases, there are clearly more space efficient algorithms, but they are harder to describe and verify. In applications, it is clear that space complexity is going to be a real issue to be dealt with.}}

Dynamic epistemic logic (DEL) allows analyzing knowledge change in a systematic way. The static component of a situation is represented by an epistemic model, while the structure of the dynamic component is encoded in an \emph{action model}. An action model can be applied to the epistemic model via so-called \emph{product update} operation, resulting in a new up-to-date epistemic model of the situation after the action has been executed. 
\journal{\tb{I'm not sure I understand the purpose of the rest of this paragraph.} \nina{It to the explain the role of logic in all this, with respect to pre and post conditions. Maybe we can do it later, but let's remove it.}} 
A language, interpreted on epistemic models, allows expressing conditions under which an action takes effect (so-called preconditions), and the effects of such actions (so-called postconditions). This setting is particularly useful for modeling the process of epistemic planning (see~\cite{bola.ea:epis,ande.ea:cond}): one can ask which sequence of actions should be executed in order for a given epistemic formula to hold in the epistemic model after the actions are executed. 

The purpose of this paper is to investigate possible learning mechanisms involved in discovering the `internal structure' of actions on the basis of their executions. In other words, we are concerned with qualitative learning of action models on the basis of observations of pairs of the form (initial state, resulting state). We analyze learnability of action models in the context of two learning conditions: finite identifiability  (conclusively inferring the appropriate action model in finite time) and identifiability in the limit (inconclusive convergence to the right action model). The paper draws on the results from formal learning theory applied to DEL (see \cite{Gie10,Gierasimczuk:2012aa,Gierasimczuk:2013aa}). 

Learning of action models is highly relevant in the context of epistemic planning. A planning agent might not initially know the effects of her actions, so she will initially not be able to plan to achieve any goals. However, if she can learn the relevant action models through observing the effect of the actions (either by executing the actions herself, or by observing other agents), she will eventually learn how to plan. Our ultimate goal is to integrate learning of actions into (epistemic) planning agents. In this paper, we seek to lay the foundations for this goal by studying learnability of action models from streams of observations. 

\journal{We will analyze learning methods in terms of time and space complexities.\tb{I don't think we should analyse in terms of space complexities, since it will make our approach look very naive (we clearly use much more space than necessary).} \nina{Yes, we should unless we analyze also other possible algorithms. And then compare.} We will also touch upon the issue of minimal action descriptions. \tb{Do we?} \nina{In the and we do not :).}We will keep the proposed learning algorithms as simple as possible, even though this implies that some of them have much worse space complexities than more efficient algorithms solving the same problem.}

The structure of the paper is as follows. In Section \ref{sec:actiontypes} we recall the basic concepts and notation concerning action models and action types in DEL. In Section \ref{sec:learning} we specify our learning framework and provide general learnability results. In Section \ref{sec:update} we study particular learning functions, which proceed via updating action models with new information. \full{Finally, in Section \ref{sec:library} we indicate how to lift our results from the level of individual action learning to that of action library learning.} In the end we briefly discuss related and further work. 

\section{Languages and action types} \label{sec:actiontypes}
Let us first present the basic notions required for the rest of the article (see \cite{balt.ea:logi,ditm.ea:sema} for more details).
Following the conventions of automated planning, we take the set of atomic propositions and the set of actions to be finite. Given a finite set $P$ of atomic propositions, we define the (single-agent) \emph{epistemic language} over $P$, $\epislang(P)$, by the following BNF: $\phi ::=  p ~|~ \neg \phi ~|~ \phi \land \phi ~|~ K\phi,$
where $p \in P$. The language $\proplang(P)$ is the propositional sublanguage without the $K\phi$ clause. When $P$ is clear from the context, we write $\epislang$ and $\proplang$ instead of $\epislang(P)$ and $\proplang(P)$, respectively. By means of the standard abbreviations we introduce the additional symbols $\to$, $\vee$, $\leftrightarrow$, $\bot$, and $\top$.

\begin{definition}[Epistemic models and states]
An \emph{epistemic model} over a set of atomic propositions $P$ is $\m = (W,R,V)$, where $W$ is a finite set of \emph{worlds}, $R\subseteq W \times W$ is an equivalence relation, called the \emph{indistinguishability relation}, and $V: P \to \mathcal{P}(W)$  is a \emph{valuation function}. An \emph{epistemic state} is a pointed epistemic model $(\m,w)$ consisting of an epistemic model $\m = (W,R,V)$ and a distinguished world $w \in W$ called the \emph{actual world}. 
\end{definition}
A \emph{propositional state} (or simply \emph{state}) over $P$ is a subset of $P$ (or, equivalently, a propositional valuation $\nu: P \to \{0,1 \}$).
We identify propositional states and singleton epistemic models via the following canonical isomorphism. A propositional state $s \subseteq P$ is isomorphic to the epistemic model $\m = (\{w\},\{(w,w)\},V)$ where $V(p) = \{w\}$ if $p \in s$ and $V(p) = \emptyset$ otherwise. Truth in epistemic states $(\m,w)$ with $\m = (W,R,V)$ (and hence propositional states) is defined as usual and hence omitted.

Dynamic epistemic logic (DEL) introduces the concept of an action model for modelling the changes to states brought about by the execution of actions \cite{balt.ea:logi}. We here use a variant that includes postconditions~\cite{ditm.ea:sema,bola.ea:epis}, which means that actions can have both epistemic effects (changing the beliefs of agents) and ontic effects (changing the factual states of affairs).
\begin{definition}[Action models] \label{defi:actionmodel}
An \emph{action model} over a set of atomic propositions $P$ is $\a = (E,Q,pre,post)$, where
$E$ is a finite set of \emph{events};
$Q \subseteq E \times E$ is an equivalence relation called the \emph{indistinguishability relation};
$pre: E \to \epislang(P)$ assigns to each event a \emph{precondition};
$post: E \to \proplang(P)$ assigns to each event a \emph{postcondition}. Postconditions are conjunctions of literals (atomic propositions and their negations) or $\top$.\footnotemark
\ $\Dom(\a) = E$ denotes the domain of $\a$. The set of all action models over $P$ is denoted $\actionspace(P)$. 
\end{definition}
\footnotetext{We are here using the postcondition conventions from~\cite{bola.ea:epis}, which are slightly non-standard. Any action model with standard postconditions can be turned into one of our type, but it might become exponentially larger in the process~\cite{ditm.ea:sema,bola.ea:epis}.} 
Intuitively, events correspond to the ways in which an action changes the epistemic state, and the indistinguishability relation codes (an agent's) ability to recognize the difference between those different ways. 
In an event $e$, $pre(e)$ specifies what conditions have to be satisfied for it to take effect, and $post(e)$ specifies its outcome.

\begin{example}\label{exam:hospital1}
Consider the action of tossing a coin. It can be represented by the following action model ($h$ means that the coin is facing heads up):
\[
 \a \ =  \quad \hspace{-2mm}
\begin{array}{l}
  \begin{tikzpicture}[>=stealth',every loop/.style={->},minimum size=1.5mm,every node/.style={auto},show background rectangle,background rectangle/.style={fill=black!10,rounded corners}]
    \node[label={below:$e_1\!:\langle \top, h \rangle$},w] (w1) at (0,0) {}; 
    \node[label={below:$e_2\!:\langle \top, \neg h \rangle$},w] (w2) at (3,0) {}; 
\end{tikzpicture} 
\end{array}
\]
We label each event by a pair whose first argument is the event's precondition while the second is its postcondition. 
Hence, formally we have $\a = (E,Q,pre,post)$ with $E = \{e_1,e_2\}$, $Q$ is the identity on $E$, $pre(e_1) = pre(e_2) = \top$, $post(e_1) = h$ and $post(e_2) = \neg h$. 
The action model encodes that tossing the coin will either make $h$ true ($e_1$) or $h$ false ($e_2$).
\end{example}

\begin{definition}[Product update]
Let $\m = (W,R,V)$ and $\a = (E,Q,pre,post)$ be an epistemic model and action model (over a set of atomic propositions $P$), respectively. The \emph{product update} of $\m$ with $\a$ is the epistemic model $\m \otimes \a = (W',R',V')$, where
 $W' = \{ (w,e) \in W \times E ~|~ (\m, w) \models pre(e) \}$;
 $R' = \{ ((w,e),(v,f)) \in W' \times W' ~|~ wRv \text{ and } eQf \}$;
$V'(p) = \{(w,e) \in W' ~|~ post(e) \models p \text{ or }( (\m,w) \models p \text{ and } post(e) \not\models \neg p )  \}$.
  For $e \in \Dom(\a)$, we define $\m \otimes e = \m \otimes (\a \upharpoonright \{ e \})$.
\end{definition}
The product update $\m \otimes \a$ represents the result of executing the action $\a$ in the state(s) represented by $\m$.
\begin{example}\label{exam:hospital2}
Continuing Example~\ref{exam:hospital1}, consider a situation of an agent seeing a coin lying heads-up, i.e., the singleton epistemic state $\M = (\{ w\} , \{w,w \} ,V)$ with $V(h) = \{ w \}$. Let us now calculate the result of executing the coin toss in this model. 
\[
  \M \otimes \a =  \quad \hspace{-3mm}
\begin{array}{l}
  \begin{tikzpicture}[>=stealth',every loop/.style={->},minimum size=1.5mm,every node/.style={auto},show background rectangle,background rectangle/.style={fill=black!10,rounded corners}]
    \node[label={below:$(w_1,e_1)\!:h$},w] (w1) at (0,0) {}; 
    \node[label={below:$(w_1,e_2)\!:\text{ }$},w] (w2) at (3,0) {}; 
\end{tikzpicture} 
\end{array}
\]
Here each world is labelled by the propositions being true at the world.
\journal{Ann opening the letter (action $\A$) in her ``belief state'' $\M$:
\[
  \M \otimes \a =  \quad \hspace{-3mm}
\begin{array}{l}
  \begin{tikzpicture}[>=stealth',every loop/.style={->},minimum size=1.5mm,every node/.style={auto},show background rectangle,background rectangle/.style={fill=black!10,rounded corners}]
    \node[label={below:$(w_1,e_1)\!:sick,anxious$},w] (w1) at (0,0) {}; 
    \node[label={below:$(w_2,e_2)\!:$},w] (w2) at (4,0) {}; 
\end{tikzpicture} 
\end{array}
\]
This model encodes that after she has opened the letter, she will know whether she is sick or not (the indistinguishability edge is gone), and be anxious if and only if sick. }
\end{example}
We say that two action models $\a_1$ and $\a_2$ are \emph{equivalent}, written $\a_1 \equiv \a_2$, if for any epistemic model $\m$, $\m \otimes \a_1 \bisim \m \otimes \a_2$, where $\bisim$ denotes standard bisimulation on epistemic models~\cite{sietsma2013action}. 

\full{\subsection{Action types}
We can identify a number of different action types. }
\begin{definition}[Action types] \label{defi:actiontypes}
An action model $\a = (E,Q,pre,post)$ is: \begin{itemize}
  \item \emph{atomic} if $| E | = 1$. 
   \item \emph{deterministic} if all preconditions are mutually inconsistent, that is, $\models pre(e) \land pre(f) \to \bot$ for all distinct $e,f \in E$. 
   \item \emph{fully observable} if $Q$ is the identity relation on $E$. Otherwise it is \emph{partially observable}.
   \item \emph{precondition-free} if $pre(e) = \top$ for all $e \in E$.
   \item \emph{propositional} if $pre(e) \in \proplang$ for all $e \in E$. 
   \item \emph{universally applicable} if $\models \bigvee_{e \in E} pre(e)$. 
   \item \emph{normal} if for all propositional literals $l$ and all $e \in E$, $pre(e) \models l$  implies $post(e) \not\models l$.
   \item with \emph{basic preconditions} if all $pre(e)$ are conjunctions of literals (propositional atoms and their negations). 
  \item with \emph{maximal preconditions} if all $pre(e)$ are maximally consistent conjunctions of literals (i.e., preconditions are conjunctions of literals in which each atomic proposition $p$ occurs exactly once, either as $p$ or as $\neg p$).
\end{itemize}
\end{definition}
Some of the notions defined above are known from existing literature~\cite{bola.ea:epis,ditm.ea:sema,sadzik2006exploring}. The newly introduced notions are precondition-free, universally applicable, and normal actions, as well as actions with basic preconditions. Note that action types interact with each other, atomic actions are automatically both deterministic and fully observable, and precondition-free actions can only be deterministic if atomic.\footnote{The actions considered in propositional STRIPS planning (called \emph{set-theoretic planning} in \cite{ghal.ea:auto}) correspond to epistemic actions that are atomic and have basic post-conditions.}

In the remainder of this section we set a uniform representation of action models that we will later on use in learning methods. We also specify and justify the restrictions we impose on action models.
\full{

\subsubsection{Propositionality}} In this paper we are concerned with product updates of \emph{propositional} states with \emph{propositional} action models. Let $s$ denote a propositional state over $P$, and let $\a = (E,Q,pre,post)$ be any propositional action model. Using the definition above and the canonical isomorphism between propositional states and singleton epistemic states, we get that $s \otimes \a$ is isomorphic to the epistemic model $(W',R',V')$, where
 $W' = \{ e \in E ~|~ s \models pre(e) \}$,
 $R' = \{ (e,f) \in W' \times W' ~|~ eQf \}$,
 $V'(p) = \{e \in W' ~|~  post(e) \models p \text{ or (} s \models p \text{ and } post(e) \not\models \neg p \text{)} \}$.
If $\a$ is fully observable, then the indistinguishability of $s \otimes \a$ is the identity relation. This means that we can think of $s \otimes \a$ as a set of propositional states (via the canonical isomorphism between singleton epistemic models and propositional states). In this case we write $s' \in s \otimes \a$ to mean that $s'$ is one of the propositional states in $s \otimes \a$.  When $\a$ is atomic we have $s \otimes a = s'$ for some propositional state $s'$ (using again the canonical isomorphism). \journal{If $\a$ is \emph{not applicable} in $s$, which we take to mean $s \not \models pre(e)$ for all $e \in \Dom(\a)$, then $s \otimes \a$ is an empty epistemic model and hence $s' \not\in s \otimes \a$ for all propositional states $s'$.}
\begin{example} \label{exam:coin1}
Consider the action model $\a$ of Example~\ref{exam:hospital1} (the coin toss). It is a precondition-free, fully observable, non-deterministic action. Consider an initial propositional state $s = \{ h \}$. Then $s \otimes \a$ is the epistemic model of Example~\ref{exam:hospital2}. It has two worlds, one in which $h$ is true, and another in which $h$ is false. 
So we have $\emptyset, \{ h \} \in s \otimes \a$, i.e., the outcome of tossing the coin is either the propositional state where $h$ is false ($\emptyset$) or the one where $h$ is true ($\{h\}$). 
\end{example}

\full{

\subsubsection{Basic preconditions and normality} 
When preconditions are basic, pre- and postconditions are of the same simple normal form, they are conjunctions of literals. Below we show that any propositional action model can be turned into an action model having this normal form. We also show that we can ensure all action models to be normal.
}

\journal{Propositional actions are sometimes called \emph{factual}~\cite{sadzik2006exploring}, but here we wish to avoid the potential confusion with the use of word \emph{factual} to talk about postconditions (postconditions provide \emph{factual}/\emph{ontic} change).}
\journal{\tb{I'm putting basic preconditions and normality into the same proposition, primarily to save space. In principle, normality doesn't require propositionality, but I think it is a minor thing in this paper, and we need to save as much space as we can. By the way, I made the page size a bit longer to fit even more...}}
\begin{proposition} \label{prop:normalisation}
  Any propositional action model is equivalent to a normal action model with basic preconditions. 
\end{proposition} 
\full{\begin{proofsketch}
Take a propositional action model  $(E,Q,pre,post)$. We first make the preconditions basic in the following way.
Take any event $e \in E$ with precondition $\phi$. Turn $\phi$ into disjunctive normal form $\bigvee_{i \in I} \bigwedge_{j \in J} p_{ij}$. Then replace $e$ by a set of events $e_i$, $i\in I$, where $post(e_i) = post(e)$ and $pre(e_i) = \bigwedge_{j \in J} p_{ij}$. Each $e_i$ is connected by a $Q$-edge to every event $e$ was originally connected to. This is done for each event $e \in E$. It is easy to see that the resulting action model $(E',Q',pre',post')$ has basic preconditions and is equivalent to the original one. 

We now ``normalise'' $post'$ into a new mapping $post''$ in order to obtain an equivalent normal action model  $(E',Q',pre',post'')$. Note that since the action model $(E',Q',pre',post')$ has basic preconditions, the normality condition can be expressed in a particularly simple way: for all literals $l$ and all $e \in E$, if $l$ is a conjunct of $pre(e)$ then it is not a conjunct of $post(e)$. For each event $e \in E$, we now define $post''(e)$ from $post'(e)$ by deleting each conjunct of $post'(e)$ which is also a conjunct in $pre'(e)$. It is easy to see that this gives an equivalent action model: 
consider an event $e$ and a literal $l$ which is both a conjunct of $pre'(e)$ and $post'(e)$. Since $l$ is a conjunct of $pre'(e)$, $l$ has to be true for $e$ to occur. Since $l$ is a conjunct of $post'(e)$, $l$ will also be true after the event $e$ has occurred. Hence, the event $e$ does not affect the truth value of $l$, and we get an equivalent event by removing $l$ from the postcondition. 
\end{proofsketch}}
\full{
In this paper we are only going to be concerned with propositional actions, and so, due to Proposition \ref{prop:normalisation}, we can restrict attention to normal actions having basic preconditions. }

\full{
\subsubsection{Universal applicability} }The condition for being universally applicable intuitively means that the action specifies an outcome no matter what state it is applied to. In this paper we will only be concerned with universally applicable action models. \full{To understand the reason for this restriction consider the example of an action \textit{open\_door} with singleton action model $\langle \neg open \wedge \neg locked, open \rangle$, i.e., if the door is currently closed and unlocked, performing \textit{open\_door} will open it. This action model does not specify what happens if an agent attempts \textit{open\_door} when the door is either already open or is locked. We can easily fix this by adding another event to the action model, $\langle open \vee locked, \top \rangle$, expressing that if one tries to open it when locked or already open, nothing happens.  More generally, any action model $\a = (E,Q,pre,post)$ which is not universally applicable can be turned into a universally applicable action model by adding the following event: $\langle \neg \vee_{e\in E} pre(e), \top \rangle$. If an agent is learning results of an action, she should in any possible state be able to attempt executing the action, and hence the action model should specify an outcome of this attempt. For this reason, we require universal applicability.}

\section{Learning action models}\label{sec:learning}
\journal{In this paper we distinguish between \emph{reactive} and \emph{proactive learning}. Reactive learning concerns situations when the learning agent is observing (other agents') action executions in various situations, but does not itself proactively choose which actions to execute. We can talk of proactive learning when the learning agent itself chooses which actions to execute, for instance with the purpose of optimizing the speed of learning. Below we will first, in Section \ref{reactive}, consider reactive learning in various settings and on various classes of actions. In Section \ref{proactive} we will consider proactive learning.
\subsection{Reactive learning}\label{reactive}
We will study reactive learning on different levels of generality.}

\full{
First we will focus on learning an individual action, i.e., inferring semantics of a single action name. 
\journal{Then we will move to action library learning, where the meaning of several action names is inferred simultaneously.}The semantics of an action name is an action model.} In the following we will use the expressions \emph{action} and \emph{action model} interchangeably. Below we will first present general results on learnability of various types of action models, and then, in Section~\ref{sect:learning_via_update}, we study particular learning methods and exemplify them.

 \journal{\tb{Planning-based learning (Walsh-type learning) is only relevant in the proactive case.}}


%
%

We are concerned with learning fully observable actions (action models). Partially observable actions are generally not learnable in the strict sense to be defined below. Consider for instance an agent trying to learn an action that controls the truth value of a proposition $p$, but where the agent cannot observe $p$ (events making $p$ true and events making $p$ false are indistinguishable). Then clearly there is no way for that agent to learn exactly how the action works. The case of fully observable actions is much simpler. If initially the agent has no uncertainty, her ``belief state'' can be represented by a propositional state. Executing any sequence of fully observable actions will then again lead to a propositional state. So in the case of fully observable actions, we can assume actions to make transitions between propositional states.

For the rest of this section, except in examples, we fix a set $P$ of atomic propositions.
\journal{\nina{I am now wondering to what extent is it important to say explicitly that we study learning in fixed environments, and that the agent can always identify in the starting state what are the propositions involved in the action.}}

\begin{definition} \label{defi:stream}
A \emph{stream} $\stream$ is an infinite sequence of pairs $(s,s')$ of propositional states over $P$, i.e., $\stream\in (\mathcal{P}(P)\times\mathcal{P}(P))^{\omega}$. The elements $(s,s')$ of $\stream$ are called \emph{observations}. Let $\mathbb{N}:=\mathbb{N}^+\cup\{0\}$, let $\stream$ be a stream over $P$, and let $s,t \in \mathcal{P}(P)$. 
$\stream_n$ stands for the $n$-th observation in $\stream$.
$\stream[n]$ stands for the the initial segment of $\stream$ of length $n$, i.e., $\stream_0,\dots,\stream_{n-1}$. 
$\Set(\stream):=\{(x,y)~|~(x,y)\text{ is an element of } \stream\}$ stands for the set of
 all observations in $\stream$; we similarly define $set(\stream[n])$ for initial segments of streams.
\end{definition}
\begin{definition}
Let $\stream$ be a stream over $P$ and $\a$  a fully observable action model over $P$.
The stream $\stream$  is \emph{sound} with respect to $\a$ if for all $(s,s')\in\Set(\stream)$, $s' \in s \otimes \a$. The stream $\stream$ is \emph{complete} with respect to $\a$ if for all $s \subseteq P$ and all $s' \in s \otimes \a$, $(s, s') \in \Set(\stream)$. In this paper we always assume the streams to be sound and complete. For brevity, if $\stream$ is sound and complete wrt $\a$, we will write: `$\stream$ \emph{is for} $\a$'. 
\end{definition}

\full{\begin{definition}[Learning function] \label{defi:learning_function}}
A \emph{learning function} is a computable $L: (\mathcal{P}(P)\times \mathcal{P}(P))^\ast \to \actionspace(P)\cup\{{\uparrow}\}$. 
\full{\end{definition}}
In other words, a learning function takes a finite sequence of observations (pairs of propositional states) and outputs an action model or a symbol corresponding to `undecided'.

We will study two types of learning: finite identifiability and identifiability in the limit. First let us focus on \emph{finite identifiability}. Intuitively, finite identifiability corresponds to conclusive learning: upon observing some finite amount of action executions the learning function outputs, with certainty, a correct model for the action in question (up to equivalence). This certainty can be expressed in terms of the function being once-defined: it is allowed to output an action model only once, there is no chance of correction later on. Formally, we say that a learning function $L$ is \emph{(at most) once defined} if for any stream $\stream$ for an action over $P$ and $n,k \in \Nat$ such that $n\neq k$, we have that $L(\stream[n]){=}{\uparrow}$ or $L(\stream[k]){=}{\uparrow}$.

\begin{definition}
Let $\X$ be a class of action models and $\A\in \X$, $L$ be a learning function, and $\stream$ be a stream. We say that:
\begin{enumerate}
\item $L$ finitely identifies $\a$ on $\stream$ if $L$ is once-defined and there is an $n\in\Nat$ s.t.\ $L(\stream[n]) \equiv \a$.
\item $L$ finitely identifies $\A$ if $L$ finitely identifies $\A$ on every stream for $\A$.
\item $L$ finitely identifies $\X$ if $L$ finitely identifies every $\A\in\X$.
\item $\X$ is finitely identifiable if there is a function $L$ which finitely identifies $\X$.
 \end{enumerate}
\end{definition}
The following definition and theorem are adapted from \cite{Muk92,LZ92,Gierasimczuk:2012aa}.
\begin{definition} \label{def_dftt}
Let $\X\subseteq\actionspace(P)$. A set $D_\a\subseteq \mathcal{P}(P)\times \mathcal{P}(P)$ is a definite finite tell-tale set $($DFTT$\,)$ for $\A$ in $\X$ if
\begin{enumerate}
\item $D_\a$ is sound for $\a$ (i.e., for all $(s, s')\in D_\a$, $s' \in s\otimes \a$),
\item $D_\a$ is finite, and
\item for any $\a'\in\X$, if $D_\a$ is sound for $\a'$, then $\a\equiv \a'$. 
\end{enumerate}
\end{definition}

\begin{lemma}\label{lemma_dftt}
$\X$ is finitely identifiable iff there is an effective procedure $\D:\X \rightarrow \mathcal{P}(\mathcal{P}(P)\times \mathcal{P}(P))$, given by $\a\mapsto D_\a$, that on input $\A$ produces a definite finite tell-tale of $\A$.
\end{lemma}
\full{
\begin{proof}

[$\Rightarrow$] Assume that $\X$ is finitely identifiable. Then there is a computable function $L$ that finitely identifies $\X$. We use that function to define $\D$. Once the learning function $L$ identifies an action $\a$ it has to give it as a definite output, and this will happen for some $\stream[n]$. We then set $\D(\a)=\Set(\stream[n])$. It is easy to check that such $\D(\a)$ is a definite tell-tale set.
[$\Leftarrow$] Assume that there is an effective procedure $\D:\X \rightarrow \mathcal{P}(\mathcal{P}(P)\times \mathcal{P}(P))$, that on input $\A$ produces a definite finite tell-tale of $\A$. Take an enumeration of $\X$ an take any $\a\in \X$ and any $\stream$ for $\a$. We use $\D$ to define the learning function. At each step $n\in \Nat$, $L$ compares $\stream[n]$ with $\D(\a_1),\ldots, \D(\a_n)$. Once, at some step $\ell\in\Nat$, it finds $\a_k$ such that $\D(\a_k)\subseteq\Set(\stream[\ell])$, it outputs $\a_k$. It is easy to verify that then $\a_k\equiv \a$.
\end{proof}}
\journal{\tb{I guess it would be a bit sad, but can we omit the proof of Lemma 1, and instead refer to the proofs it is adopted from? I guess there shouldn't be any significant differences to the proof in other settings? The problem currently is that the more principled learning theoretic approach we are now taking might end up making the paper longer rather than shorter, because quite a lot of machinery has to be introduced. It might still be worth it, and there is certainly no turning back now, but we should try to keep it as short as possible.}}

In other words, the finite set of observations $D_\a$ is consistent with only one action $\a$ in the class (up to equivalence of actions). $\D$ is a computable function that gives a $D_\a$ for any action $\A$.

\begin{theorem}\label{atomic_fin_id}
For any finite set of propositions $P$ the set of (fully observable) deterministic propositional actions over $P$ is finitely identifiable.
\end{theorem}
\full{
\begin{proof}
We use Lemma \ref{lemma_dftt}, and hence define $\D$: 
$\D(\A)=\{(s,s')~|~ s\otimes \a=s' \text{, where } s, s'\in \mathcal{P}(P)\}.$
Let us check that indeed $\D(\A)$ is a DFTT for $\A$. We need to show conditions 1, 2 and 3 of Definition~\ref{def_dftt} for $\D(\A)$.
1: $\D(\A)$ is sound for $\a$, trivially. 
2: $\D(\A)$ is finite, because $P$ is finite. 
3: Let us take any propositional action $\a'$ such that $\D(\a)$ is sound for $\a'$. This means, by the definition of $\D$ above and the fact that $\a$ and $\a'$ are deterministic, that for all propositional states $s, s'$ over $P$, if $s \otimes \a = s'$ then $s \otimes \a' = s'$. It follows that $s \otimes \a = s' \otimes \a'$ for all propositional states $s$, and hence $\a \equiv \a'$ (since $\a$ and $\a'$ are propositional).
Finally, $\D$ is computable because $P$ is finite.
\end{proof}
}

\begin{example} \label{exam:coin2}
Theorem~\ref{atomic_fin_id} shows that deterministic actions are finitely identifiable. We will now show that this does not carry over to non-deterministic actions, that is, non-deterministic actions are in general not finitely identifiable.
Consider the action of tossing a coin, given by the action model $\a$ in Example~\ref{exam:hospital1}. If in fact the coin is fake and it will always land tails (so it only consists of the event $e_2$), in no finite amount of tosses the agent can exclude that the coin is fair, and that heads will start appearing in the long run (that $e_1$ will eventually occur). So the agent will never be able to say ``stop'' and declare the action model to only consist of $e_2$. This argument can be generalised, leading to the theorem below.
\end{example}

\begin{theorem}\label{atomic_not_fin_id}
For any finite set of propositions $P$ the set of arbitrary (including non-deterministic) fully observable propositional actions over $P$ is not finitely identifiable.
\end{theorem}
\full{
\begin{proof}
Assume that the set of arbitrary propositional actions over $A$ is finitely identifiable. Then there is a learning function $L$ that finitely identifies it. Among such actions we will have two, $\a$ and $\a'$, such that $\a'=\a\upharpoonright \D(\a')$.\footnote{For any action model $\a = (E,Q,pre,post)$ and any subset $E' \subseteq E$ we define $a \upharpoonright E'$ as the restriction of $\a$ to the domain $E'$, that is, $a \upharpoonright E' = (E',Q',pre',post')$  where $Q' = Q \cap(E')^2$, $pre' = pre \upharpoonright E'$ and $post' = post \upharpoonright E'$.} Let us now construct a stream $\stream$ on which $L$ fails to finitely identify one of them. Let the $\stream$ start with enumerating all pairs of propositional states that are sound for the smaller action, $\a'$, and keep repeating this pattern. Since this is a stream for $\a'$ indeed the learning function has to at some point output an equivalent of $\a'$ (otherwise it fails to finitely identify $\a'$, which leads to contradiction). Assume that this happens at some stage $n\in\Nat$. Now, observe that $\stream[n]$ is sound with respect to $\a$ too, so starting at the stage $n+1$ let us make $\stream$ enumerate the rest of remaining pairs of propositional states consistent with $\a$. That means that there is a stream $\stream$ for $\a$ on which $L$ does not finitely identify $\a$. Contradiction.
\end{proof}
}



A weaker condition of learnability, \emph{identifiability in the limit}, allows widening the scope of learnable actions, to cover also the case of arbitrary actions. Identifiability in the limit requires that the learning function after observing some finite amount of action executions outputs a correct model (up to equivalence) for the action in question and then forever keeps to this answer (up to equivalence) in all the outputs to follow. This type of learning can be called `inconclusive', because certainty cannot be achieved in finite time.

\begin{definition}
Let $\X$ be a class of action models and $\A\in \X$, $L$ be a learning function, and $\stream$ be a stream. We say that:
\begin{enumerate}
\item $L$ identifies $\a$ on $\stream$ in the limit if there is $k\in\Nat$ such that for all $n\geq k$, $L(\stream[n]) \equiv \a$.
\item $L$ identifies $\A$ in the limit if $L$ identifies $\A$ in the limit on every $\stream$ for $\A$.
\item $L$ identifies $\X$ in the limit if $L$ identifies in the limit every $\A\in\X$.
\item $\X$ is identifiable in the limit if there is an $L$ which identifies $\X$ in the limit.
 \end{enumerate}
\end{definition}

The following theorem is adapted from \cite{Ang80}.
\begin{theorem}\label{atomic_lim_id}
For any finite set of propositions $P$ the set of (fully observable) propositional actions over $P$ is identifiable in the limit.
\end{theorem}
\full{
\begin{proof}
The argument is similar to the proof of Theorem \ref{atomic_fin_id}. Analogously to the concept of definite finite tell-tale set, we define a weaker notion of finite tell-tale set (FTT).
Let $P$ be a set of propositions and let $\X\subseteq\actionspace(P)$. A set $D_\a\subseteq \mathcal{P}(P)\times \mathcal{P}(P)$ is a finite tell-tale set $($FTT$\,)$ for $\A$ in $\X$ if:
1. $D_\a$ is sound for $\a$ (i.e., for all $(s, s')\in D_\a$, $s' \in s\otimes \a$);
2. $D_\a$ is finite, and
3. for any $\a'\in\X$, if $D_\a$ is sound for $\a'$, then $\a\equiv\a'\upharpoonright X$, where $X\subseteq \Dom(\a')$.

Similarly to the argument for Lemma \ref{lemma_dftt} one can show that
$\X$ is identifiable in the limit iff there is an effective procedure $\D:\X \rightarrow \mathcal{P}(\mathcal{P}(P)\times \mathcal{P}(P))$, given by $\a\mapsto D_\a$, that on input $\A$ enumerates a finite tell-tale of $\A$. We will omit the proof for the sake of brevity.

Now it is enough to show that indeed such a function $\D$ can be given for the set of arbitrary (fully observable, propositional) actions over $P$. Define 
$\D(\A)=\{(s,s')~|~ s' \in  s\otimes \a \text{, where } s, s'\in \mathcal{P}(P)\}$.
Let us check that indeed $\D(\A)$ is a FTT for $\A$. 1: $\D(\A)$ is sound for $\a$, trivially. 
2: $\D(\A)$ is finite, because $P$ is finite.
3: Let us take any propositional action $\a'$ such that $\D(\a)$ is sound for $\a'$. This means, by the definition of $\D$ above that for all propositional states $s, s'$ over $P$, if $s' \in s\otimes \a$ then $s' \in s\otimes \a'$. This implies that $s \otimes \a$ is a submodel of $s \otimes \a'$ for all propositional states $s$, and hence that $\a$ is equivalent to a submodel of $\a'$ (since actions are propositional). 

Finally, again $\D$ is computable because $P$ is finite.
\end{proof}
}

Having established the general facts about finite identifiability and identifiability in the limit of propositional fully-observable actions, we will now turn to studying particular learning methods suited for such learning conditions. 

\section{Learning actions via update} \label{sect:learning_via_update}\label{sec:update}


Standard DEL, and in particular public announcement logic, deals with learning within epistemic models.  If an agent is in a state described by an epistemic model $\m$ and learns from a reliable source, that $\phi$ is true, her state will be updated by eliminating all the worlds where $\phi$ is false. That is, the model $\m$ will be restricted to the worlds where $\phi$ is true. This can also be expressed in terms of action models, where the learning of $\phi$ corresponds to taking the product update of $\m$ with the event model $\langle \phi, \top \rangle$ (public announcement of $\phi$). 

Now we turn to learning \emph{actions} rather than learning \emph{facts}. Actions are represented by action models, so to learn an action means to infer the action model that describes it. Consider again the action model $\a$ of Example~\ref{exam:hospital1}.
The coin toss is non-deterministic and fully observable: either $h$ or $\neg h$ will non-deterministically be made true and the agent is able to distinguish these two outcomes (no edge between $e_1$ and $e_2$). However, we can also think of $\a$ as the hypothesis space of a \emph{deterministic} action, that is, the action $\a$ is in fact deterministically making $h$ true or false, but the agent is currently uncertain about which one it is. Given the prior knowledge that the action in question must be deterministic, learning the action could proceed in a way analogous to that of update in the usual DEL setting. 

It could for instance be that the agent knows that the coin is fake and always lands on the same side, but the agent initially does not know which.
\journal{$\A$ is an action representing the 
push of a light switch at a doorway between two rooms, but the agent does not know whether the switch turns on the light in the current room ($p := \top$), or the other room ($q := \top$).} 
After the agent has executed the action once, she will know. She will observe either $h$ becoming false or $h$ becoming true, and can hence discard either $e_1$ or $e_2$ from her hypothesis space. 
She has now \emph{learned} the correct deterministic action model for tossing the fake coin. 
\journal{Assume she executes action named $a$ in the propositional state $\emptyset$ where neither $p$ nor $q$ is true (light is off in both rooms), and that she afterwards observes that $p$ has become true (the updated model is $\{p \}$). Then she knows that the action $a$ is described by $e_1$ and that $e_2$ can be discarded from her hypothesis space. So she will update her action model by restricting $\a$ to the domain $\{e_1\}$.}
 Note the nice symmetry to learning of facts: here, learning of facts means eliminating worlds in epistemic models, learning of actions means eliminating events in action models.

\emph{In the rest of this section, all action models are silently assumed to be: fully observable, propositional, and universally applicable. Furthermore, we can assume them to be normal and have basic preconditions, due to Proposition~\ref{prop:normalisation}.}

\subsection{Learning precondition-free atomic actions}
We will first propose and study an update learning method especially geared towards learning the simplest possible type of ontic actions: precondition-free atomic actions.

\begin{definition}
For any deterministic action model $\a$ and any pair of propositional states $(s,s')$, the \emph{update} of $\a$ with $(s,s')$ is defined by $\a ~|~ (s,s') := \a \upharpoonright \{ e \in E ~|~ \text{if $pre(e) \models s$ then $s \otimes  e  = s'$} \}$. 
For a set $S$ of pairs of propositional states, we define: 
$\a ~|~ S := \a \upharpoonright \{ e \in E ~|~ \text{for all $(s,s') \in S$, if $pre(e) \models s$ then $s \otimes e = s'$}\}.$	
\end{definition}
The update $\a ~|~ (s,s')$ restricts the action model $\a$ to the events that are consistent with observing $s'$ as the result of executing the action in question in the state $s$.
\journal{\tb{I had to redefine the above. It only worked for precondition-free actions. When we have preconditions it is different. Consider an observation $(s,s')$ and an event $e$ with $s \not \models pre(e)$. We can not use the observation $(s,s')$ to exclude $e$ as a possible event, since its precondition is not even satisfied. Even with my generalisation above, the update mechanism still only works for deterministic actions. Well, it is *defined* for any type of action, but doesn't preserve any interesting properties in the case of non-deterministic actions. So I restricted the def to deterministic actions.}}
\journal{
\tb{I guess we don't need this now?:
\begin{definition}Let $\a$ be an action model and $e_1,e_2 \in \Dom(\a)$. An observation $(s,s')$ is called a \emph{separating observation} for $e_1,e_2$ if $s \otimes e_1 = s'$ and $s \otimes e_2 \neq s'$ (or vice versa). Two events are called \emph{separable}  if there exists a separating observation for them. \tb{So I guess this will become obsolete? But we can keep it until the new proofs are completely in place.}
\end{definition}
Note that if $(s,s')$ is a separating observation for a pair of events $e_1,e_2$ of $\a$ then $\a ~|~ (s,s')$ will only contain one of these events. }
}
\begin{definition} \label{defi:l1}
The \emph{update learning function for precondition-free atomic actions} over $P$ is the learning function $L_1$ defined by $L_1(\stream[n]) = \a_{init}^1 ~|~ \Set(\stream[n])$
where $\a_{init}^1 = (E,Q,pre,post)$ with 
$E = \{ \psi ~|~ \text{$\psi$ is a consistent conjunction of literals over $P$} \}$; 
$Q$ is the identity relation on $E \times E$;
$pre(e) = \top$ for all $e \in E$;
$post(\psi) = \psi$.
\end{definition}
In Figure \ref{rosettas} we show a generic example of such update learning for $P=\{p,q\}$.
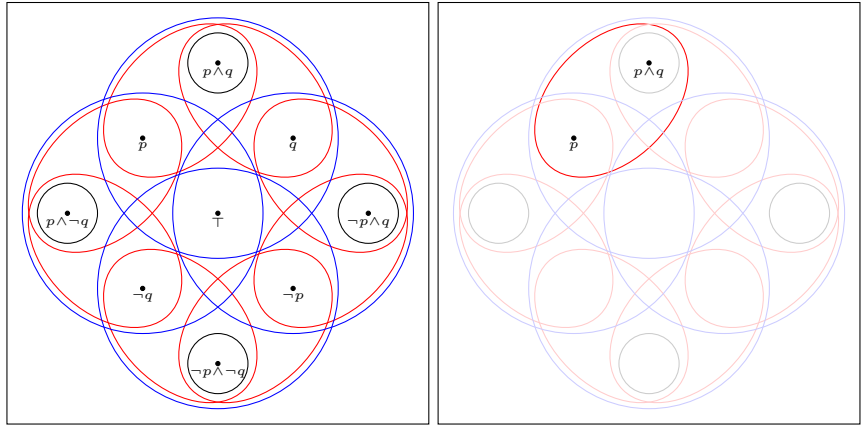
\begin{figure} 
\def\Aset{(0,5) ellipse (1 and 1)}
\def\Bset{(0,-5) ellipse (1 and 1)}
\def\Cset{(5,0) ellipse (1 and 1)}
\def\Dset{(-5,0) ellipse (1 and 1)}

\def\Eset{ellipse (3 and 1.5) } 

\def\Fset{(-2,0) ellipse (3 and 1.5)} 
\def\Gset{(0,5) ellipse (1 and 1)}
\def\Hset{(0,-5) ellipse (1 and 1)}
\def\Iset{(5,0) ellipse (1 and 1)}
\def\Jset{(-5,0) ellipse (1 and 1)}
\def\Kset{(2,2)ellipse (3 and 1.5) } 
\def\Lset{(-2,0) ellipse (3 and 1.5)} 
\tikzstyle{P} = [draw,black]
\tikzstyle{Pfade}=[draw,black!20!white]
\tikzstyle{O} = [draw,blue]
\begin{center}
\begin{tikzpicture}[scale=0.40, framed]
\tikzset{
    >=stealth'}    
\draw[fill=black,color=black] (0,5) node[below] {\tiny{$p{\wedge}q$}} circle (2pt);    
\draw[fill=black,color=black] (-2.5,2.5) node[below] {\tiny{$p$}} circle (2pt);
\draw[fill=black,color=black] (0,0) node[below] {\tiny{$\top$}} circle (2pt);
\draw[fill=black,color=black] (2.5,2.5) node[below] {\tiny{$q$}} circle (2pt);
\draw[fill=black,color=black] (5,0) node[below] {\tiny{$\neg p{\wedge} q$}} circle (2pt);
\draw[fill=black,color=black] (-5,0) node[below] {\tiny{$p{\wedge} \neg q$}} circle (2pt);
\draw[fill=black,color=black] (-2.5,-2.5) node[below] {\tiny{$\neg q$}} circle (2pt);
\draw[fill=black,color=black] (2.5,-2.5) node[below] {\tiny{$\neg p$}} circle (2pt);
\draw[fill=black,color=black] (0,-5) node[below] {\tiny{$\neg p{\wedge} \neg{q}$}} circle (2pt);
       \path[P] \Aset ;
        \path[P] \Bset;
         \path[P] \Cset ;
        \path[P] \Dset;
\draw[rotate around={-45:(-1.25,3.75)},red] (-1.25,3.75)ellipse (2 and 3);       
 \draw[rotate around={45:(-1.25,-3.75)},red] (-1.25,-3.75)ellipse (2 and 3);  
  \draw[rotate around={-45:(1.25,-3.75)},red] (1.25,-3.75)ellipse (2 and 3);  
   \draw[rotate around={45:(1.25,3.75)},red] (1.25,3.75)ellipse (2 and 3);  
    \draw[rotate around={45:(3.75, 1.25)},red] (3.75, 1.25)ellipse (2 and 3);  
     \draw[rotate around={-45:(-3.75, 1.25)},red] (-3.75, 1.25)ellipse (2 and 3);  
      \draw[rotate around={-45:(3.75,-1.25)},red] (3.75,-1.25)ellipse (2 and 3);  
       \draw[rotate around={45:(-3.75, -1.25)},red] (-3.75, -1.25)ellipse (2 and 3);          
       
           \draw[blue] (0, 2.5) ellipse (4 and 4);  
                      \draw[blue] (0, -2.5) ellipse (4 and 4);  
                                 \draw[blue] (2.5, 0) ellipse (4 and 4);  
                                            \draw[blue] (-2.5,0) ellipse (4 and 4);

\end{tikzpicture}
\begin{tikzpicture}[scale=0.40, framed]
\tikzset{
    >=stealth'}    
\draw[fill=black,color=black] (0,5) node[below] {\tiny{$p{\wedge}q$}} circle (2pt);    
\draw[fill=black,color=black] (-2.5,2.5) node[below] {\tiny{$p$}} circle (2pt);
       \path[Pfade] \Aset \Bset \Cset \Dset;
\draw[rotate around={-45:(-1.25,3.75)},red] (-1.25,3.75)ellipse (2 and 3);       
 \draw[rotate around={45:(-1.25,-3.75)},red!20!white] (-1.25,-3.75)ellipse (2 and 3);  
  \draw[rotate around={-45:(1.25,-3.75)},red!20!white] (1.25,-3.75)ellipse (2 and 3);  
   \draw[rotate around={45:(1.25,3.75)},red!20!white] (1.25,3.75)ellipse (2 and 3);  
    \draw[rotate around={45:(3.75, 1.25)},red!20!white] (3.75, 1.25)ellipse (2 and 3);  
     \draw[rotate around={-45:(-3.75, 1.25)},red!20!white] (-3.75, 1.25)ellipse (2 and 3);  
      \draw[rotate around={-45:(3.75,-1.25)},red!20!white] (3.75,-1.25)ellipse (2 and 3);  
       \draw[rotate around={45:(-3.75, -1.25)},red!20!white] (-3.75, -1.25)ellipse (2 and 3);          
       
           \draw[blue!20!white] (0, 2.5) ellipse (4 and 4);  
                      \draw[blue!20!white] (0, -2.5) ellipse (4 and 4);  
                                 \draw[blue!20!white] (2.5, 0) ellipse (4 and 4);  
                                            \draw[blue!20!white] (-2.5,0) ellipse (4 and 4);

\end{tikzpicture}
\caption{On the left hand side $\a^1_{init}$ for $P=\{p,q\}$, together with sets corresponding to possible observations. We have labelled each event $e$ by $post(e)$. On the right hand side the state of learning after observing $\stream_0=(\{q\}, \{p,q\})$.}\label{fig_lex_not_scons}\label{rosettas}
\end{center}
\end{figure}

\begin{theorem} \label{theorem:learningdeterministicpreconditionfree}
The class of precondition-free atomic actions is finitely identifiable by the update learning function $L^{update}_1$, defined in the following way:
\[
  L^{update}_1(\stream[n]) = 
  \begin{cases}
  L_1(\stream[n]) & \text{if } card(\Dom(L_1(\stream[n])))=1\\
  & \text{and for all } k< n, \ L^{update}_1(\stream[k])=\ \uparrow;\\
  \uparrow & otherwise.
  \end{cases}
\]
\end{theorem}  

\full{
\begin{proof}
Let $\a$ denote a precondition-free atomic action over $P$, and let $\stream$ be a stream for $\a$. 
We show that $L^{update}_1$ finitely identifies $\a$ on $\stream$. First note that $L^{update}_1$ is obviously (at most) once-defined. Further, we need to show that
for some $n\in\Nat$, $L^{update}_1(\stream[n]) \equiv \a$. By definition of $L^{update}_1$, it is the case only if $card(\Dom(\a_{init}^1 \mid \Set(\stream[n]))) =1$ and $(\a_{init}^1 \mid \Set(\stream[n])) \equiv \a$. 

Since $\a$ is atomic and precondition-free, it must consist of a single event of the form $\langle \top, \psi \rangle$. By definition of $\a^1_{init}$ (Definition~\ref{defi:l1}), this implies that there is an event $e \in \Dom(\a^1_{init})$ such that $\a = \a^1_{init} \upharpoonright \{ e \}$. Since $\stream$ is a stream for $\a$, $e$ is in $\a_{init}^1 \mid \Set(\stream[n])$.

By Theorem \ref{atomic_fin_id} we know that $\a$ is finitely identifiable, so, by Lemma \ref{lemma_dftt}, there is a DFTT $D_\a$ for $\A$. Since $\stream$ is for $\a$ and $D_\a$ is by definition finite and sound for $\a$, there is $n\in\Nat$ such that $D_{\a}\subseteq \Set(\stream[n])$. By definition of DFTT and of $\A^1_{init}$ we get that for all $e' \in \A^1_{init}$ such that $e'\neq e$ there is $(s,s')$ in $\stream[n]$ such that $s\otimes e'\neq s'$, and hence $e'$ is not in $\a^1_{init} \mid \Set(\stream[n])$. To see why that is assume the contrary, i.e., that there is an $e' \in \A_{init}$, $e' \neq e$ such that  for all $(s,s') \in \stream[n]$ we have $s \otimes e' = s'$. Then $D_\a$ is also sound for the singleton action model containing only $e'$. But this contradicts that $D_\a$ is a DFTT for $\a$ (since all pairs of distinct events in $\A^1_{init}$ are inequivalent).
 
Combining the above we get that $\A^1_{init}~|~ \Set(\stream[n])$ contains exactly one event, $e$, and hence $\a_{init}^1 \mid \Set(\stream[n]) = \a_{init}^1 \upharpoonright \{ e\} = \a$, showing the required.
\end{proof}
}


\subsection{Learning deterministic actions with preconditions}
We now turn to learning of action models with preconditions. First we only treat the case of maximal preconditions, then afterwards we generalise to arbitrary (not necessarily maximal) preconditions.

\journal{
RECONSIDER THIS, CF. REVIEWS.
\begin{lemma} \label{lemma:maximallyconsistent}
  Let $\a$ be an action model over $P$ with basic preconditions. Then there exists an action model $\a'$ equivalent to $\a$ in which all preconditions are maximally consistent conjunctions of literals (i.e.,\ preconditions are conjunctions of literals in which each atomic proposition $p$ occurs exactly once, either as $p$ or as $\neg p$). If $\a$ is deterministic, so is $\a'$.
\end{lemma}
\begin{proof}
  Consider any event $e = \langle \phi, \psi \rangle$ of $\a$ where $p_1,\dots,p_n$ denote the propositions from $P$ that do not occur in $\phi$. Let $\Gamma$ be the set of maximally consistent conjunctions of literals from $\{p_1,\dots,p_n\}$. Then $e$ can be replaced by the set of events $\{ \langle \phi \wedge \gamma, \psi \rangle ~|~ \gamma \in \Gamma \}$, where any other event $e'$ originally connected to $e$ is now connected to all of the new events. These new events all have preconditions that are maximally consistent conjunctions of literals (recall that action models are assumed to have basic preconditions). It is easy to check that the modification gives an action model equivalent to the original. By performing this modification for each event $e$ of $\a$ we reach the required action model $\a'$. If $\a$ is deterministic, it means that all preconditions of $\a$ are mutually inconsistent, and it is clear that all preconditions of $\a'$ will be mutually inconsistent as well (any pair of distinct maximally consistent conjunctions is mutually inconsistent).
\end{proof}
}

\begin{definition}\label{defi:updatelearning2}
The \emph{update learning function for deterministic action models with maximal preconditions} over $P$ is the learning function $L_2$ defined by
  $L_2( \stream[n]) = \a_{init}^2 ~|~ \Set(\stream[n])$
where $\a_{init}^2 = (E,Q,pre,post)$ with 
$E = \{ (\phi, \psi) ~|~$ $\phi$ is a maximally consistent conjunction of literals over $P$ and 
$\psi$ is a conjunction of literals over $P$ not containing any of the conjuncts of $\phi$ $\}$; 
   $Q$ is the identity on $E \times E$;
$pre((\phi,\psi)) = \phi$; 
$post((\phi,\psi)) = \psi$.
\end{definition}
\begin{theorem} \label{theorem:identifiabilitydeterministic}
  The class of deterministic action models with maximal preconditions is finitely identifiable by the following update learning function $L^{update}_2$. 
  \[
  L^{update}_2(\stream[n]) = 
  \begin{cases}
  L_2(\stream[n]) & \text{if for all } e, e'\in \Dom(L_2(\stream[n])) \\
  & \text{if } e\neq e', \text{ then } pre(e)\neq pre(e')\\
  & \text{and for all } k< n, \ L^{update}_2(\stream[k])=\ \uparrow;\\
  \uparrow & otherwise.
  \end{cases}
\]
\end{theorem}  
\full{
\begin{proof}
Consider any event $e = \langle \phi, \psi \rangle$ in $\a$. Its precondition $\phi$ is a maximally consistent conjunction of literals over $P$. Due to normality, its postcondition $\psi$ can not contain any of the conjunctions of $\phi$. Hence $e$ must be identical to one of the events of $\a_{init}^2$. In other words, $\a$ must be isomorphic to a restriction/submodel of $\a_{init}^2$. 

Let $\stream$ be a stream for $\a$. We show that $L^{update}_2$ finitely identifies $\a$ on $\stream$. $L^{update}_2$ is obviously (at most) once-defined. Further, we need that for some $n\in\Nat$, $L^{update}_2(\stream[n]) \equiv \a$. By Theorem \ref{atomic_fin_id} we know that $\a$ is finitely identifiable, so, by Lemma \ref{lemma_dftt}, there is a DFTT $D_\a$ for $\A$, and hence there is $n\in\Nat$ such that $D_{\a}\subseteq \Set(\stream[n])$. Firstly, note that all distinct $e, e'\in \Dom(L_2(\stream[n]))$ have distinct preconditions $pre(e)\neq pre(e')$. Here a similar argument applies as in the proof of Theorem \ref{theorem:learningdeterministicpreconditionfree}. This gives us that $L^{update}_2$ is guaranteed to give an answer. Secondly, we have to show that for any $k<n$, if for all distinct $e, e'\in \Dom(L_2(\stream[k]))$ it is the case that $pre(e)\neq pre(e')$, then $L_2(\stream[k])\equiv \a$. Take any $k<n$, there are two cases. If $\stream[k]$ does not include a DFTT for $\a$, then there are two distinct $e,e'\in L_2(\stream[k])$ such that $pre(e)=pre(e')$. If $\stream[k]$ includes a DFTT for $\a$, then actions $\a'\not\equiv \a$ have been eliminated from $\A^2_{init}$ by step $k$, and hence $L_2(\stream[k])\equiv \a$.

\end{proof}
}

\begin{example} \label{exam:pushbuttonandlightbulb}
Consider a simple scenario with a pushbutton and a light bulb. Assume there is only one proposition $p$: `the light is on', and only one action: pushing the button. We assume an agent wants to learn the functioning of the pushbutton. There are 4 distinct possibilities: 1) the button does not affect the light (i.e., the truth value of $p$); 2) it is an \emph{on button}: it turns on the light unconditionally (makes $p$ true); 3) it is an \emph{off button}: it turns off the light unconditionally (makes $p$ false); 4) it is an \emph{on/off button} (flips the truth value of $p$). If the agent is learning by update, it starts with the action model $\a_{init}^2$ containing the following events: $\langle p, \top \rangle$, $\langle \neg p, \top \rangle$, $\langle p, \neg p  \rangle$, and $\langle \neg p, p \rangle$.
Note that by definition $\a_{init}^2$ does not contain the events $\langle p, p \rangle$ and $\langle \neg p, \neg p \rangle$, since they both have a postcondition conjunct which is also a precondition conjunct. Assume the first two observations the learner receives (the first elements of a stream $\stream$) are $(\emptyset, \{p \})$ and $(\{ p \}, \emptyset)$. Since the agent uses learning by update, she revises her model as follows (cf.\ Definition~\ref{defi:updatelearning2}): 
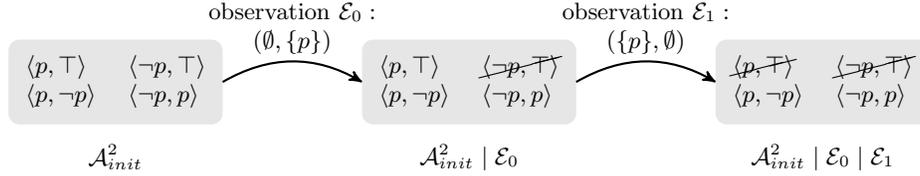
\[ \small
    \begin{tikzpicture}[auto, scale=0.94]
  \node[color=black!10,fill,draw,rounded corners,inner sep=1.7mm] (n1) at (0,0) {\color{black}
  $
  \begin{array}{ll} 
    \langle p, \top \rangle & \langle \neg p, \top \rangle \\ 
     \langle p, \neg p \rangle \quad & \langle \neg p, p \rangle
  \end{array}$
  };
  \node[color=black!10,fill,draw,rounded corners,inner sep=1.7mm] (n2) at (5,0) {\color{black}
    $\begin{array}{ll}
    \langle p, \top \rangle & \cancel{\langle \neg p, \top \rangle} \\ 
     \langle p, \neg p\rangle \quad & \langle \neg p, p \rangle
  \end{array}$  
  };
  \node[color=black!10,fill,draw,rounded corners,inner sep=1.7mm] (n3) at (10,0) {\color{black}
    $\begin{array}{ll}
   \cancel{ \langle p, \top \rangle} & \cancel{ \langle \neg p, \top \rangle} \\ 
    \langle p, \neg p \rangle \quad & \langle \neg p, p \rangle
  \end{array}$  
  }; 
  \node[below of=n1] (lab1) {$\a_{init}^2$};
  \node[below of=n2] (lab2) {$\a_{init}^2 \mid \stream_0$};
  \node[below of=n3] (lab3) {$\a_{init}^2 \mid \stream_0 \mid \stream_1$};
   \draw[->,thick, >=stealth'] (n1.east) to[bend left] node[auto] {\begin{tabular}{cc} observation $\stream_0:$ \\ $(\emptyset,\{p\})$ \end{tabular} }  (n2.west); 
 \draw[->,thick,>=stealth'] (n2.east) to[bend left] node[auto] {\begin{tabular}{cc} observation $\stream_1:$ \\ $(\{p\},\emptyset)$ \end{tabular} }  (n3.west); 
 \end{tikzpicture}
  \]
Now the agent has reached a deterministic action model $\a_{init}^2 \mid \Set(\stream[2])$, and can report this to be the correct model of the action, cf.\ Theorem~\ref{theorem:identifiabilitydeterministic}. 
Note that the two observations correspond to first pushing the button when the light is off ($\stream_0$), and afterwards pushing the button again after the light has come on ($\stream_1$). These two observations are sufficient to learn that the pushbutton is of the on/off type (it has one event that makes $p$ true if $p$ is currently false, and another event making $p$ true if currently false). 

Consider now another stream $\stream'$ where the first two elements are $(\emptyset, \{p \})$ and $(\{p \}, \{ p \})$.
\short{Update learning will now instead reduce the initial action model $\a_{init}^2$ to the action model only containing $\langle p, \top \rangle$ and $\langle \neg p, p \rangle$.}
 \full{Update learning will now work as follows:
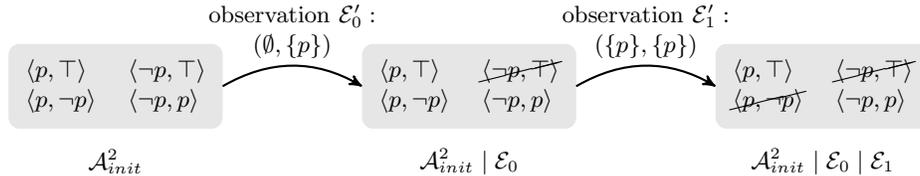
\[ \small
    \begin{tikzpicture}[auto, scale=0.94]
  \node[color=black!10,fill,draw,rounded corners,inner sep=1.7mm] (n1) at (0,0) {\color{black}
  $
  \begin{array}{ll} 
    \langle p, \top \rangle & \langle \neg p, \top \rangle \\ 
     \langle p, \neg p  \rangle \quad & \langle \neg p, p \rangle
  \end{array}$
  };
  \node[color=black!10,fill,draw,rounded corners,inner sep=1.7mm] (n2) at (5,0) {\color{black}
    $\begin{array}{ll}
    \langle p, \top \rangle & \cancel{\langle \neg p, \top \rangle} \\ 
     \langle p, \neg p \rangle \quad & \langle \neg p, p  \rangle
  \end{array}$  
  };
  \node[color=black!10,fill,draw,rounded corners,inner sep=1.7mm] (n3) at (10,0) {\color{black}
    $\begin{array}{ll}
   \langle p, \top \rangle & \cancel{ \langle \neg p, \top \rangle} \\ 
    \cancel{\langle p, \neg p \rangle} \quad & \langle \neg p, p  \rangle
  \end{array}$  
  }; 
  \node[below of=n1] (lab1) {$\a_{init}^2$};
  \node[below of=n2] (lab2) {$\a_{init}^2 \mid \stream_0$};
  \node[below of=n3] (lab3) {$\a_{init}^2 \mid \stream_0 \mid \stream_1$};
   \draw[->,thick,>=stealth'] (n1.east) to[bend left] node[auto] {\begin{tabular}{cc} observation $\stream'_0:$ \\ $(\emptyset,\{p\})$ \end{tabular} }  (n2.west); 
 \draw[->,thick,>=stealth'] (n2.east) to[bend left] node[auto] {\begin{tabular}{cc} observation $\stream'_1:$ \\ $(\{p\}, \{ p \})$ \end{tabular} }  (n3.west); 
 \end{tikzpicture}
  \]}
This time the learner identifies the button to be an on button, again after only two observations. It is not hard to show that in a setting with only one propositional symbol $p$, any deterministic action will be identified after having received the first two distinct observations. 
\end{example}
\begin{example}
  Consider learning the functioning of an $n$-bit binary counter, where the action to be learned is the increment operation. For $i=1,\dots,n$, we use the proposition $c_i$ to denote that the $i$th least significant bit is 1. Consider first the case $n=2$. A possible stream for the increment operation is the following: 
\newcommand{\bincount}[4]{\small \setlength{\arraycolsep}{0.5mm} \begin{array}{|c|c|c|c|c|} \cline{1-2} \cline{4-5} #1 &  #2 & \to & #3 & #4 \\ \cline{1-2} \cline{4-5}  \multicolumn{1}{c}{c_2} & \multicolumn{1}{c}{c_1} & \multicolumn{1}{c}{}&  \multicolumn{1}{c}{c_2} & \multicolumn{1}{c}{c_1} \end{array}} 
  \[
   \setlength{\arraycolsep}{2mm}
    \begin{array}{cccccc}
      (\emptyset, \{ c_1 \}), & (\{ c_1 \}, \{ c_2 \}), & (\{ c_2 \}, \{ c_2, c_1 \}),  & (\{ c_2, c_1 \}, \{ \emptyset \}), & \cdots  \\[3mm]
      \bincount{0}{0}{0}{1} & \bincount{0}{1}{1}{0}& \bincount{1}{0}{1}{1} & \bincount{1}{1}{0}{0}  & \cdots 
    \end{array}
  \]
Using the update learning method on this stream, it is easy to show that the learner will after the first 4 observations be able to report the correct action model containing the following events:
  $\langle \neg c_2 \wedge \neg c_1, c_1 \rangle, \langle \neg c_2 \wedge c_1, c_2 \wedge \neg c_1 \rangle, \langle c_2 \wedge \neg c_1, c_1 \rangle, 
   \langle c_2 \wedge c_1, \neg c_2 \wedge \neg c_1 \rangle$.
Note that since $\a_{init}^2$ has maximal preconditions, the action model learned for an $n$-bit counter will necessarily contain $2^n$ events: one for each possible configuration of the $n$ bits. If we did not insist on maximal preconditions, we would only need $n+1$ events to describe the $n$-bit counter: $\langle \neg c_{i} \wedge c_{i-1} \wedge c_{i-2} \wedge \cdots \wedge c_1,  c_{i} \wedge \neg c_{i-1} \wedge \neg c_{i-2} \wedge \cdots \wedge \neg  c_1 \rangle$ for all $i=2,\dots,n$, $\langle \neg c_1, c_1 \rangle$ and $\langle c_n \wedge \cdots \wedge c_1, \neg c_n \wedge \cdots \wedge \neg c_{1}\rangle$. This means that there is room for improvement in our learning method. 

\end{example}
To allow learning of deterministic action models where preconditions are not required to be maximal we need a different learning condition. Consider learning an action on $P = \{ p\}$ that sets $p$ true unconditionally. With non-maximal preconditions, all of the following events would be consistent with any stream for the action: $\langle \top, p \rangle, \langle \neg p, p \rangle, \langle p,\top  \rangle$. To get to a deterministic action model, the learning function would have to delete either the first or the two latter events. We can make it work as described in the following.

For any action model $\a$ we define 
\[
\min(\a) = \a \upharpoonright \{ e ~|~ \text{there is no event $e' \neq e$ with $pre(e) \models pre(e')$} \}.
\]
Furthermore, we define $L_3$ to be exactly like $L_2$ of Definition~\ref{defi:updatelearning2} except in the definition of $E$, $\phi$ can be any conjunction of literals, not only maximally consistent ones. 
\begin{theorem} \label{theorem:identifiabilitydeterministicgeneral}
  The class of deterministic action models is finitely identifiable by the following update learning function $L^{update}_3$. 
  \[
  L^{update}_3(\stream[n]) = 
  \begin{cases}
  \min(L_3(\stream[n])) 
  &\text{if for all $s \in \mathcal{P}(P)$ there exists an $s'$ s.t.}\ \\
  &(s,s') \in \Set(\stream[n])  \text{ and for all } k< n, \\
  & L^{update}_3(\stream[k])=\ \uparrow;\\
  \uparrow & otherwise.
  \end{cases}
\]
\end{theorem}  
\full{The proof of this theorem is left out.} 
The theorem can be seen as a generalisation of Theorem~\ref{theorem:identifiabilitydeterministic} in that it allows the learner to learn more compact action models in which maximal consistency of preconditions is not enforced (on the contrary, by the way the $\min$ operator is defined above, the learner will learn an action model with \emph{minimal} preconditions). For instance, in the case of the $n$-bit counter considered in Example~\ref{theorem:identifiabilitydeterministic}, it can be shown that the learner will learn the action model with $n+1$ events instead of the one with $2^n$ events.

\full{
\section{Action library learning}\label{sec:library}
In this section we introduce action library learning. This is the type of learning most relevant in planning. A finite set of action names is available to the agent. In order to plan a sequence of actions towards a goal it is essential to know what the corresponding actions do. As most of the results in this section are straightforward generalizations of our previous results, for the sake of space we will omit all proofs.

An action library corresponds to what is called a \emph{planning domain} (and sometimes also an \emph{action library}) in classical planning: a specification of the available actions and their action schemas. \emph{Action library learning} is the learning problem where the agent is initially only given a set $\actionnames$ of action names and has to learn the action library $\actlib: \actionnames \to \actionspace(P)$. That is, the agent initially only knows the names of the available actions, and it then learns the action models that correspond to those names.
\begin{definition}[Action library]
Let $P$ denote a set of atomic propositions, and let $\actionnames$ denote a finite set, the set of \emph{action names}. An \emph{action library} over $P,\actionnames$ is a mapping $\actlib: \actionnames \to \actionspace(P)$ (a mapping from action names into action models). If all actions in the codomain of $\actlib$ enjoy property $X$, then $\actlib$ is called an \emph{$X$ action library} (e.g., a \emph{deterministic action library} $\actlib$ is one where all action models in the codomain of $\actlib$ are deterministic).
\end{definition}
\journal{
\begin{definition}[Equivalence of action libraries]
 Two action libraries $\actlib_1: \actionnames \to \actionspace(P)$ and $\actlib_2: \actionnames \to \actionspace(P)$ are called \emph{equivalent}, written $\actlib_1 \equiv \actlib_2$, if $\actlib_1(a) \equiv \actlib_2(a)$ for all $a \in \actionnames$. They are called \emph{equivalent from} $s$, written $\actlib_1 \equiv_s \actlib_2$, where $s$ is a propositional state $s$, if
 \[
   s \otimes \actlib_1(a_1) \otimes \cdots \otimes \actlib_1(a_n) \equiv s \otimes \actlib_2(a_1) \otimes \cdots \otimes \actlib_2(a_n)
 \]
 for all sequences of action names $a_1,\dots,a_n \in \actionnames$.
\end{definition}
Action library equivalence between $\actlib_1$ and $\actlib_2$ from $s$ means that if the initial state of a system is $s$, then no matter which of the named actions are executed in this system, $\actlib_1$ and $\actlib_2$ will specify the same resulting state (up to bisimulation). In other words, the action descriptions are equivalent with respect to the states reachable from $s$. 
}

\journal{\tb{Insert something about universal applicability. Why it is OK to restrict attention to such actions.}}

Streams and learning functions for action libraries are defined similarly to the case of individual actions. Let $P$ be a set of atomic propositions, and $\actionnames$ a set of action names. A \emph{stream} over $P,\actionnames$ is an infinite sequence of triples $(s,a,s')$ where $s,s'$ are propositional states over $P$ and $a \in \actionnames$. Notations $\stream_n$, $\stream[n]$, $\Set(\stream)$ and $\Set(\stream[n])$ are defined similarly to Definition~\ref{defi:stream}. Given $a \in A$, the $a$-\emph{substream} of $\stream$ is given by $\stream^a = \{ (s,s') ~|~ (s,a,s') \in \Set(\stream) \}$. Let $\actlib$ be an action library over $P,\actionnames$. A \emph{stream} for $\actlib$ is an infinite sequence of triples $(s,a,s')$ where $s,s'$ are propositional states over $P$, $a \in \actionnames$ and $s'  \in s \otimes \actlib(a)$. A \emph{library learning function} over $P,\actionnames$ is a mapping $L: (\mathcal{P}(P)\times A \times \mathcal{P}(P))^\ast \to ((\actionnames \to \actionspace(P)) \cup \{\uparrow \})$.
\journal{
\begin{definition}
    A stream $\stream$ is \emph{complete} wrt.\  $\actlib$ if for all propositional states $s$ over $P$ and all action names $a \in \actionnames$, $(s,a, s\otimes \actlib(a)) \in \Set(\stream)$.  Given $a \in A$, the $a$-\emph{substream} of $\stream$ is given by $\stream_a = \{ (s,s') ~|~ (s,a,s') \in \Set(\stream) \}$. 
  \end{definition}
\begin{definition}[Learning function]
 A \emph{learning function} over $P,\actionnames$ is a mapping $L: (\mathcal{P}(P)\times A \times \mathcal{P}(P))^\ast \to ((\actionnames \to \actionspace(P)) \cup \{\uparrow\})$.
\end{definition}
}
Given a learning function $L$ for individual actions over $P$ (Definition~\ref{defi:learning_function}), we define the \emph{induced library learning function} $\mathsf{L}$ over $P,\actionnames$ by 
\begin{enumerate}
  \item $\mathsf{L}(\stream) := \ \uparrow$ if $L(\stream^a) =\ \uparrow$ for some $a \in \actionnames$.
  \item Otherwise, for all $a \in \actionnames$, $\mathsf{L}(\stream)(a) := L(\stream^a)$.
\end{enumerate}
From Theorem~\ref{theorem:identifiabilitydeterministicgeneral} we then immediately get the following.
\begin{theorem} 
  The class of deterministic action libraries is finitely identifiable by the library learning function $\mathsf{L}^{update}_3$ (induced from $L^{update}_3$ of Theorem~\ref{theorem:identifiabilitydeterministicgeneral}).
\end{theorem}  

\full{

\begin{example} \label{exam:electric1} Consider the electrical circuit below consisting of two switches ($1$ and $2$), a voltage source (left) and a light bulb (right). 
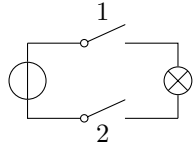
\[\begin{tikzpicture}[circuit ee IEC, set make contact graphic = var make contact IEC graphic]
    \draw (0,0) to [make contact={info={$1$}}] ++(right:2) to [bulb] ++(down:1); 
    \draw (0,0) to [voltage source] ++(down:1) to [make contact={info' = {$2$}}] ++(right:2);
  \end{tikzpicture}
\]
When both switches are closed, the light will be on, otherwise it will be off. Let proposition $s_i$ denote that switch $i$ is closed and let $l$ denote that the light is on. Assume the available actions are $flip_1$ and $flip_2$ that flip switch $1$ and $2$, respectively. Consider an agent trying to learn how the switches and the circuit work. This agent then tries to learn an action library over $\{ s_1, s_2, l \}$, $\{ flip_1, flip_2 \}$. Given a stream $\stream$ for the action library, it can be shown that the learning function $\mathsf{L}^{update}_3$ will eventually return the following action library that describes it (note that it can be described in many equivalent ways).
\[
 \begin{array}{l}
   l(flip_1) = \langle \neg s_1{\wedge} \neg  s_2, s_1 {\wedge} \neg l \rangle, \langle \neg s_1 {\wedge} s_2, s_1 {\wedge} l \rangle,  \langle s_1 {\wedge} \neg s_2, \neg s_1 {\wedge} \neg l \rangle, \langle s_1 {\wedge} s_2, \neg s_1 {\wedge} \neg l \rangle  \\
   l(flip_2) = \langle \neg s_1 {\wedge} \neg  s_2, s_2 {\wedge} \neg l \rangle, \langle \neg s_1 {\wedge} s_2, \neg s_2 {\wedge} \neg l \rangle,  \langle s_1 {\wedge} \neg s_2, s_2 {\wedge} l \rangle, \langle s_1 {\wedge} s_2, \neg s_2 {\wedge} \neg l \rangle
 \end{array}
\]

\end{example}

\journal{
\begin{example}
Consider again the electrical circuit from Example~\ref{exam:electric1}. Initially the circuit is in state $s_0 = \emptyset$ (both switches are open and the light is off). 
\tb{Here I will continue the electrical network example. Two action libraries might differ in the outcome of activating switch 1 in a state where the light is on. But then they will be equivalent from the initial state where none of the switches are activated and the light is off: no state is reachable where the light is on but switch 1 not activated. (@tb: see handwritten residuum notes on Sokoban version of this example).} 
\end{example}
}

}

\journal{
\nina{We define a general method of composing a learning function of inferior ones dealing with separate action names. Then say that since all of the separate ones are fin. id. by some learning method then the whole library can be learned by using it. Next, we could also say that if we consider a heterogenous action library (containing both deterministic and non-deterministic actions) can be id. in the limit via update.} \tb{I tried to something like that above, at least a minimal version of it. I also added an example, at least sketched it. I think we should keep it at this and drop the stuff about consecutive streams. It will take quite a while to read the paper thoroughly already as it is now, so it is probably not even wise to add more, except if it was something really powerful. Going to consecutive streams is somehow powerful, in particular when relating to planning, but we are not going to be able to relate to planning except in very overall, abstract terms. And then it is probably not worth it to introduce the stuff about consecutive streams and equivalence with respect to action libraries etc. It will be quite a lot of additional machinery, and we will have to explain the reader why it is even interesting to do this generalisation. Furthermore, my ``update and restrict'' method would have to be modified to take universal applicability into account. So far it has simply been deleting the events whose preconditions turn out to be non-reachable. But this destroys universal applicability. So instead, the method would have to first delete those events but then afterwards replace them with events having empty postconditions. That easily becomes a bit technical to write down, will take up additional space, and might at first confuse the reader. No need to do that. We should rather focus on spending the rest of the time streamlining the paper a bit more, and do more to motivate, frame and defend our work. I think that will in the end be more important to whether the paper gets accepted or not.}}

\journal{

A stream $\stream$ is called \emph{consecutive} if for all $i$, if $\stream_i = (s_i,a_i,t_i)$ and $\stream_{i+1} = (s_{i+1},a_{i+1},t_{i+1})$  then $t_i = s_{i+1}$. Learnability on consecutive streams is more problematic. Consider the action library $\actlib$ on $A = \{a_1,a_2 \}, P = \{ p \}$ with $\actlib(a_1) = \actlib(a_2) = $ the action that makes $p$ true. This action library is not learnable by any consecutive stream (we can at most learn $a_1$ or $a_2$). However, some types of action libraries  are still finitely identifiable. We define a library $\actlib$ on $P,\actionnames$ to be \emph{reversible} if for all states $s$ and sequences $a_1,\dots,a_n$, if $s \otimes \actlib(a_1) \otimes \cdots \actlib(a_n) = s'$ then there exists a sequence $a'_1,\dots,a'_n$ such that $s' \otimes \actlib(a'_1) \otimes \cdots \actlib(a'_n) = s$. \tb{I changed strongly connected to reversible. The interesting examples are usually not strongly connected, because not all propositional states are accessible. This is e.g. the case with the example of two buttons controlling a light switch. It is not possible to reach a state where the light is on and one of the switches is not activated.}
\tb{Define: learning by consecutive streams. We need to be able to point out the \emph{initial state} $s_0$ of such a stream. The definition below could need some cleaning up as well.}
\begin{definition}[Learning by update and restrict]
  The \emph{update and restrict learning function} for deterministic, fully observable action libraries on streams with initial state $s_0$ is the learning function $L_4$ defined by
    \begin{align*}
      L_4(\stream[n])(a) = (\a^2_{init} ~|~ \Set(\stream_a [n])) \upharpoonright \{ e ~|~ \text{there exists $a_1,\dots,a_n$ such that} \\
      \text{$s_0 \otimes L_4(\stream[n-1])(a_1) \otimes \cdots \otimes L_4(\stream[n-1])(a_n) \models pre(e)$} \}
      \end{align*}
\end{definition}
\tb{The above has to be modified to take universal applicability into account. So instead of cutting off events whose preconditions are not reachable, these events should be replaced by events with empty postconditions.}
\tb{@Nina: The intuition in the definition above is as I explained in the Skype meeting and illustrated with the ``nice figure'': The agent only preserves the events for which it is believed that their preconditions might still be achievable. All other events are discarded. It works because the agent is always over-approximating the real action model (is being optimistic). I think it is quite important to exemplify this learning function and explain its underlying intuition.}

\begin{theorem}
The class of deterministic, fully observable and reversible action libraries is finitely identifiable by the update learning function using consecutive streams.
\end{theorem}
\begin{proof}
  \tb{To be added. The only complications compared to the previous proof is: 1) to prove that the restriction works and produces the right model in the end---and that the agent knows when the model has been learned; 2) to deal with the fact that the stream is consecutive, but reversibility makes it rather straightforward.}
\end{proof} 
\begin{corollary}
The class of deterministic, precondition-free and reversible action libraries is finitely identifiable by the update learning function using consecutive streams.
\end{corollary}

\begin{example}
  \tb{Now I can give a full formalisation of the example with two switches and a lamp (the electrical network example) in which the agent plays around with the switches until she learns what they do.}
\end{example}

Finite identifiability can be contrasted with identifiability in the limit which can solve the unlikely case of learning non-deterministic actions \nina{Here we need to think some more}.
}

\journal{
\section{Proactive learning}\label{proactive}
\tb{We will not be able to fit anything about proactive learning into this paper, I think. I recommend we don't cover it at all, but just mention it in future work. Then we should also structure the paper a bit differently, and change the introduction to learning strategies.}
\tb{In fact, any deterministic, precondition-free action $a$ is learnable in only two observations: $(\emptyset, \emptyset \otimes a), (P, P \otimes a)$. The first observation tells you for which $p$ the action has $post(e)(p) = \top$ and the second observation tells you for which $p$ the action has $post(e)(p) = \bot$ (where $e$ is the name of the event in the singleton action model $a$). But I think this point should be made in the proactive section, since it will be a proposition about identifiability by proactive learning: The class of deterministic, precondition-free actions is identifiable in 2 steps by the update learning function (proactive setting). We should remember to summerise our results in a table: reactive/proactive along one dimension, action/domain/setting along the other. Actually there is a third dimension: action types, but we can probably find a nice way to represent that in the table.} 
}

}
\section{Conclusions and related work}
This paper is the first to study the problem of learnability of action models in dynamic epistemic logic (DEL). We provided an original learnability framework and several early results concerning fully observable propositional action models with respect to conclusive (finite identifiability) and inconclusive (identifiability in the limit) learnability. Apart from those general results, we proposed various learning functions which code particular learning algorithms. Here, by implementing the update method (commonly used in DEL), we demonstrated how the learning of action models can be seen as transitioning from nondeterministic to deterministic actions.

\subsubsection{Related work} 
A similar qualitative approach to learning actions has been addressed by \cite{Walsh:2008fk} within the STRIPS planning formalism. The STRIPS setting is more general than ours in that it uses atoms of first-order predicate logic for pre- and postconditions. It is however less general in neglecting various aspects of actions which we have successfully treated in this paper: negative preconditions (i.e., negative literals as precondition conjuncts), negative postconditions, conditional effects (which we achieve through non-atomic action models). We believe that the ideas introduced here can be applied to generalize the results of \cite{Walsh:2008fk} to richer planning frameworks allowing such action types. It is also worth mentioning here that there has been quite substantial amount of work in relating DEL and learning theory (see \cite{Gie10,Gierasimczuk:2013aa} for overviews), which concerns a different setting: treating update and upgrade revision policies as long term learning methods, where learning can be seen as convergence to certain types of knowledge (see \cite{BGS11,Baltag:2014ac,Baltag:2015uo}). A study of abstract properties of finite identifiability in a setting similar to ours, including various efficiency considerations, can be found in \cite{Gierasimczuk:2012aa}.

\subsubsection{Further directions}
 In this short paper we 
only considered fully observable actions applied in fully observable states, and hence did not use the full expressive power of the DEL formalism. The latter still remains adequate, since action models provide a very well-structured and principled way of describing actions in a logical setting, and since its use opens ways to various extensions. The next steps are to cover more DEL action models: those with arbitrary pre- and postconditions, and those with partial observability and multiple agents.  As described earlier, partially observable actions are not learnable in the strict sense considered above, but  we can still investigate agents learning ``as much as possible'' given their limitations in observability. The multi-agent case is particularly interesting due to the possibility of agents with varied limitations on observability, and the possibility of communication within the learning process.

We plan to study the computational complexity of learning proposed in this paper, but also to investigate other more space-efficient learning algorithms. We are also interested algorithms that produce minimal action models. \full{For instance, if we allow action models that have event postconditions specified as mappings from propositions to formulas (as is standard in DEL), then the action library for the circuit of Example~\ref{exam:electric1} can be described using only 2 events.} \full{However, learning such minimal action descriptions might turn out to be computationally much harder.} Furthermore, we here considered only what we call \emph{reactive learning}: the learner has no influence over observations. We would also like to study the case of \emph{proactive learning}, where the learner gets to choose which actions to execute, and hence observe their effects. This is probably the most relevant type of learning for a general learning-and-planning agent. In this context, we also plan to focus on \emph{consecutive streams}: streams corresponding to executing sequences of actions rather than observing arbitrary state transitions. Our ultimate aim is to relate learning and planning within the framework of DEL. Those two cognitive capabilities are now investigated mostly in separation---our goal is to bridge them.



\journal{\tb{Is it OK that we don't relate to quantitive learning approaches (machine learning stuff)? There is a LOT of that in the context of planning, but unfortunately I don't know any of the papers.}}

\full{
\section{Acknowledgments}
Nina Gierasimczuk is funded by an Innovational Research Incentives Scheme Veni grant 275-20-043, Netherlands Organisation for Scientific Research (NWO).
We thank Martin Holm Jensen for his contributions through discussions we had in the early stages of this paper. We are also grateful to Mikko Berggren Ettienne who made us aware of \cite{Walsh:2008fk}. We are also grateful to the anonymous reviewers for their valuable feedback.
}

\bibliographystyle{splncs03.bst}

\providecommand{\noopsort}[1]{}


\end{document}